\documentclass{article}
\pdfoutput=1
\usepackage[utf8]{inputenc}
\usepackage[a4paper, margin=1in]{geometry}
\usepackage[active]{srcltx}
\usepackage{babel}
\usepackage{verbatim}
\usepackage{mathtools}
\usepackage{bm}
\usepackage{amsmath}
\usepackage{amssymb}
\usepackage{microtype,comment}
\usepackage{graphicx}
\usepackage{subfigure}
\usepackage{booktabs} 
\usepackage{hyperref}
\usepackage{verbatim}
\usepackage{mathtools}
\usepackage{bm}
\usepackage{amsmath}
\usepackage{amssymb} 
\usepackage{amsthm}
\usepackage{babel}
\usepackage{amsthm}
\usepackage{dsfont}
\usepackage{array}
\usepackage{mathrsfs}
\usepackage{color}
\usepackage[numbers]{natbib}
\usepackage{enumitem}

\def\BibTeX{{\rm B\kern-.05em{\sc i\kern-.025em b}\kern-.08em
    T\kern-.1667em\lower.7ex\hbox{E}\kern-.125emX}}
\theoremstyle{plain} 
\newtheorem{lemma}{\textbf{Lemma}} 

\newtheorem{theorem}{\textbf{Theorem}}
 
\newtheorem{assumption}{\textbf{Assumption}}

\theoremstyle{definition}
\theoremstyle{remark}

\newtheorem{remark}{\textbf{Remark}}

\newcommand{\xstar}{x^*}
\newcommand{\ystar}{y^*}
\newcommand{\RR}{\mathbb{R}}
\newcommand{\cc}{\mathfrak{c}}
\newcommand{\qfrak}{\mathfrak{q}}
\newcommand{\pfrak}{\mathfrak{p}}
\newcommand{\afrak}{\mathfrak{a}}
\newcommand{\FF}{\mathcal{F}}
\newcommand{\EE}{\mathbb{E}}

\title{$O(1/k)$ Finite-Time Bound for Non-Linear Two-Time-Scale Stochastic Approximation}
\author{Siddharth Chandak\\\\Department of Electrical Engineering, Stanford University, Stanford, USA\\\\\texttt{chandaks@stanford.edu}}

\date{April 2025}

\begin{document}
\maketitle
\begin{abstract}
Two-time-scale stochastic approximation (SA) is an algorithm with coupled iterations which has found broad applications in reinforcement learning, optimization and game control. In this work, we derive mean squared error bounds for non-linear two-time-scale iterations with contractive mappings. In the setting where both stepsizes are order $\Theta(1/k)$, commonly referred to as single time-scale SA with multiple coupled sequences, we obtain the first $\mathcal{O}(1/k)$ rate without imposing additional smoothness assumptions. In the setting with true time-scale separation, the previous best bound was $\mathcal{O}(1/k^{2/3})$. We improve this to $\mathcal{O}(1/k^a)$ for any $a<1$ approaching the optimal $\mathcal{O}(1/k)$ rate. The key step in our analysis involves rewriting the original iteration in terms of an averaged noise sequence whose variance decays sufficiently fast. Additionally, we use an induction-based approach to show that the iterates are bounded in expectation. Our results apply to Polyak averaging, as well as to algorithms from reinforcement learning, and optimization, including gradient descent-ascent and two-time-scale Lagrangian optimization.
\end{abstract}


\section{Introduction}
\label{sec:introduction}
Stochastic Approximation (SA) is a popular class of iterative algorithms used in finding the fixed point of an operator given its noisy realizations \cite{Robbins-Monro}. These algorithms have been widely studied in the past few decades due to their applications in fields such as reinforcement learning, optimization, communication networks, and stochastic control \cite{Borkar-book}.

Two-time-scale SA is a variant of SA where the fixed points for two coupled operators are found using two iterations with different rates of update. These coupled iterations have been used in various fields including reinforcement learning (RL) and optimization. For example, SSP Q-learning \cite{Abounadi} is a two-time-scale algorithm for control of average reward MDPs. Actor-critic algorithms are another example where the `actor' and the `critic' update at different rates \cite{konda1999actor}. Gradient descent-ascent algorithm for learning saddle points  \cite{minimax} is an example of two-time-scale SA in optimization. Game control algorithms can often be modeled using two time-scales, as the players update their actions on a faster time-scale and the manager updates the game parameters on a slower time-scale \cite{Chandak-game}.

Consider the following coupled iterations.
\begin{equation*}
    \begin{aligned}
&x_{k+1}=x_k+\alpha_k (f(x_k,y_k)-x_k+M_{k+1})\\
&y_{k+1}=y_k+\beta_k (g(x_k,y_k)-y_k+M'_{k+1}).
\end{aligned}
\end{equation*}
Here $x_k\in\RR^{d_1}$ and $y_k\in\RR^{d_2}$ are the iterates updated on the faster and slower time-scale, respectively. These time-scales are dictated by the stepsizes $\alpha_k$ and $\beta_k$, respectively, where $\beta_k$ decays faster than $\alpha_k$. Here, $f(\cdot)$ and $g(\cdot)$ are the non-linear Lipschitz functions whose fixed points we aim to obtain. $M_{k+1}$ and $M'_{k+1}$ denote additive and multiplicative martingale noise.

The objective of the algorithm is to find solutions to $f(\xstar,\ystar)=\xstar$ and $g(\xstar,\ystar)=\ystar$. We assume that operators $f(x,y)$ and $g(\xstar(y),y)$ are contractive with respect to $x$ and $y$, respectively. Here, $\xstar(y)$ is the fixed point for $f(\cdot,y)$. This assumption is sufficient to ensure uniqueness of $(\xstar,\ystar)$ and asymptotic convergence of the iterates to the fixed points. Such an assumption is standard in the literature on finite-time analysis of SA algorithms and is satisfied by a wide range of RL and optimization algorithms \cite{Q-learn,TD0,Nesterov}.

While classical analysis of SA schemes has focused on asymptotic analysis \cite{Kushner,Borkar-book}, there has been a recent interest in obtaining finite-time guarantees on performance of these algorithms. These finite-time guarantees can also be broadly divided into two types: high probability or concentration bounds (e.g., \cite{Chandak_conc, Chen-conc}) and mean square error bounds (e.g.,\cite{Zaiwei,Srikant,Doan}). We focus on the latter in this paper.

\subsection{Main Contributions}\label{subsec:contributions}
Our aim in this paper is to obtain mean square error bounds for non-linear two-time-scale SA, i.e., obtain bounds on  $\EE\left[\|x_k-\xstar\|^2+\|y_k-\ystar\|^2\right]$ where $\|\cdot\|$ denotes the Euclidean norm. 
The existing results on such bounds can be broadly classified into two categories. We next describe these categories, and how we improve on the existing results.
\begin{enumerate}
    \item \textbf{Single time-scale SA with multiple coupled sequences:} This is the setting where $\beta_k$ and $\alpha_k$ are both of the order $\Theta(1/k)$. In this setting, the time-scale separation is decided by $\lim_{k\uparrow\infty} \beta_k/\alpha_k=\gamma<1$. A bound of $\mathcal{O}(1/k)$ was obtained in \cite{Shen-Chen-smooth} but they assumed smoothness of the operator $\xstar(y)$. We obtain the first $\mathcal{O}(1/k)$ bound in this setting without any additional assumptions (Theorem \ref{thm:main-1/k}).
    
    \item \textbf{`True' two-time-scale SA:} This is the setting where $\lim_{k\uparrow\infty} \beta_k/\alpha_k=0$, i.e., the two-time-scales are truly separated. The previous best known bound in this setting was $\mathcal{O}(1/k^{2/3})$ obtained for stepsizes of the form $\beta_k=1/k$ and $\alpha_k=1/k^{2/3}$ \cite{Doan}. We significantly improve this bound to $\mathcal{O}(1/k^\afrak)$ obtained when stepsizes are of the form $\alpha_k=\mathcal{O}(1/k^\afrak)$ and $\beta_k=\mathcal{O}(1/k)$. Here $\afrak$ can be arbitrarily close to $1$ (Theorem \ref{thm:main-1/k^a}). An advantage of this setting is that the stepsize sequences can be chosen independently of system parameters. This allows for a more robust bound, while still achieving a bound of $\mathcal{O}(1/k^\afrak)$ where $\afrak$ can be arbitrarily close to one.
\end{enumerate}

We identify that the reason for a bound of $\mathcal{O}(1/k^{2/3})$ in \cite{Doan} was the manner in which $M'_{k+1}$, the martingale noise in the slower time-scale iteration, was handled. In fact, in \cite{Chandak-TTS-CDC}, a bound of $\mathcal{O}(1/k)$ is obtained for a special case where the slower time-scale is noiseless. In this work, we introduce an averaged noise sequence and auxiliary iterates for the slower time-scale. The key intuition behind introducing these is to transform the original iteration into one in which the noise sequence has a decaying variance, instead of the constant variance as in the original iteration. We emphasize that analyzing this averaged noise sequence is just a proof technique. Neither does our algorithm have an additional averaging step, nor do we give bounds on averaged iterates.

Concretely, we introduce the averaged noise sequence $U_{k+1}=\beta_kM'_{k+1}+(1-\beta_k)U_k$ (with $U_0=0$) and define the iterates $z_k=y_k-U_k$. We analyze the iterations for $x_k$ and $z_k$, obtaining a bound of $\mathcal{O}(1/k)$ on $\EE\left[\|z_k-\ystar\|^2\right]$. Moreover, we show that $\EE\left[\|y_k-z_k\|^2\right]$ decays at a rate of $\mathcal{O}(1/k)$. This allows us to obtain a rate of $\mathcal{O}(1/k)$ on the mean square error. 

 \subsection{Outline and Notation}
This paper is structured as follows: Section \ref{sec:formulation} sets up the problem and gives examples of algorithms that fit into our framework. Section \ref{sec:main} presents the main results of this paper after defining the stepsize sequences. Section \ref{sec:outline} contains a proof sketch for our main results. Section \ref{sec:conc} concludes the paper and presents some future directions.

Throughout this work, $\|\cdot\|$ denotes the Euclidean norm, and $\langle x_1,x_2\rangle$ denotes the inner product given by $x_1^Tx_2$.

 \subsection{Related Works}
Finite-time bounds for two-time-scale SA have primarily focused on the linear case, i.e., where the functions $f(\cdot)$ and $g(\cdot)$ are linear. In \cite{Konda}, an asymptotic rate of $\mathcal{O}(1/k)$ is obtained. This has been extended to a finite-time bound in \cite{Shaan, Kaledin}. While our results can also be applied to the linear case to get a bound of $\mathcal{O}(1/k)$, these prior works exploit the structure of linearity to obtain bounds which hold for a larger set of stepsizes. For example, in \cite{Shaan}, it is shown that a rate of $\mathcal{O}(1/k)$ can be obtained on $\EE[\|y_k-\ystar\|^2]$ for the stepsize sequences $\alpha_k=\mathcal{O}(1/k^\afrak)$ and $\beta_k=\mathcal{O}(1/k)$, where $\afrak\in(0.5,1)$. In \cite{Dalal}, high probability guarantees on the performance of linear two-time-scale SA have been obtained.

As discussed in Subsection \ref{subsec:contributions}, the existing works on  non-linear two-time-scale SA can be divided into two categories. We improve on the existing results in both categories.

\textbf{`True' time-scale separation:}\quad This is the setting where $\lim_{k\uparrow\infty}\beta_k/\alpha_k=0$. Among these works, the closest to us in terms of generality of assumptions is \cite{Doan}. They obtained a bound of $\mathcal{O}(1/k^{2/3})$ by choosing $\alpha_k=\alpha/(k+2)^{2/3}$ and $\beta_k=\beta/(k+2)$. We significantly improve their result to obtain a bound of $\mathcal{O}(1/k^\afrak)$ where $\afrak$ can be arbitrarily close to $1$. Moreover, we consider multiplicative martingale difference noise instead of the additive noise that they consider. This allows us to consider cases where the noise scales affinely with the iterates. This enables us to incorporate linear SA with noisy matrices. We employ an induction-based approach to deal with the multiplicative martingale noise and show that the iterates are bounded in expectation. The analysis in \cite{Doan} was extended in \cite{Chandak-TTS-CDC}, where a bound of $\mathcal{O}(1/k^{2/3})$ was obtained in the setting where the operators are contractive under arbitrary norms. Unlike above works, a rate of $\mathcal{O}(1/k)$ is obtained in \cite{Han-linearity} but they assume local linearity of the operators. This allows them to borrow techniques from the analysis of linear SA.

\textbf{SA with multiple coupled sequences:} These works obtain the rate $\mathcal{O}(1/k)$ using stepsizes of the form $\Theta(1/k)$ in both time-scales. Among these, \cite{Doan-1/k} obtains a rate of $\mathcal{O}(1/k)$ by modifying the algorithm to include an additional averaging step. In \cite{Chandak-TTS-CDC}, a bound of $\mathcal{O}(1/k)$ is shown for the special case where the slower time-scale is noiseless. In \cite{Shen-Chen-smooth}, an additional assumption is taken that states that $\xstar(y)$ is differentiable and smooth. In this setting, the work closest to us in terms of assumptions on the operator and the fixed point, is \cite{1/k-const-step}. This work also obtains a bound of $\mathcal{O}(1/K)$ in the general non-linear two-time-scale setting, but they work with a fixed horizon $K$, constant stepsizes which depend on the horizon, and provide a bound that holds only for $k=K$.


\section{Problem Formulation}\label{sec:formulation}
We set up the notation and assumptions in this section. Consider the following coupled iterations.
\begin{equation}\label{iter-main}
    \begin{aligned}
&x_{k+1}=x_k+\alpha_k (f(x_k,y_k)-x_k+M_{k+1})\\
&y_{k+1}=y_k+\beta_k (g(x_k,y_k)-y_k+M'_{k+1}).
\end{aligned}
\end{equation}
Here $x_k\in\RR^{d_1}$ is the iterate updating on the faster time-scale and $y_k\in\RR^{d_2}$ is the iterate updating on the slower time-scale. Their rates of update are dictated by the stepsizes $\alpha_k$ and $\beta_k$, respectively. We discuss these in more detail in the next section. $M_{k+1}$ and $M'_{k+1}$ are martingale difference noise sequences (Assumption \ref{assu:Martingale}). Our first assumption is the key contractive assumption on function $f(\cdot,\cdot)$.

\begin{assumption}\label{assu:f-contrac}
The function $f(x,y):\RR^{d_1}\times \RR^{d_2}\mapsto \RR^{d_1}$is $\lambda$-contractive in $x$ for any $y\in\RR^{d_2}$, i.e., 
    $$\|f(x_1,y)-f(x_2,y)\|\leq \lambda\|x_1-x_2\|,$$
for all $x_1,x_2\in\RR^{d_1}$ and $y\in\RR^{d_2}$. Here $\lambda<1$ is the contraction factor.
\end{assumption}

Using the Banach contraction mapping theorem, the above assumption implies that $f(\cdot,y)$ has a unique fixed point for each $y$. We denote this fixed point by $\xstar(y)$, i.e., for each $y$, there exists unique $\xstar(y)$ such that $f(\xstar(y),y)=\xstar(y)$. We next state the contractive assumption on function $g(\xstar(\cdot),\cdot)$.
\begin{assumption}\label{assu:g-contract}
The function $g(\xstar(\cdot),\cdot):\RR^{d_2}\mapsto\RR^{d_2}$ is $\mu$-contractive, i.e., 
    $$\|g(\xstar(y_1),y_1)-g(\xstar(y_2),y_2)\|\leq \mu\|y_1-y_2\|,$$
for $y_1,y_2\in\RR^{d_2}$. Here $\mu$ is the contraction factor.
\end{assumption}
This assumption implies the existence of a unique fixed point for $g(\xstar(\cdot),\cdot)$, i.e., there exists unique $\ystar$ such that $g(\xstar(\ystar),\ystar)=\ystar$. We also define $\xstar\coloneqq \xstar(\ystar)$. Our goal is to study the convergence rate of $x_k$ to $\xstar$ and of $y_k$ to $\ystar$. The following remark comments on an alternate formulation.

\begin{remark}\label{remark-alternate}
    We formulate our problem as finding the fixed point for two contractive operators. Alternatively, this can be formulated as a root finding problem for functions $\tilde{f}(x,y)$ and $\tilde{g}(x,y)$. The necessary assumptions in this case are that functions $\tilde{f}(\cdot,y)$ and $\tilde{g}(\xstar(\cdot),\cdot)$ are strongly monotone. The formulations are equivalent by defining $f(x,y)=x-\zeta\tilde{f}(x,y)$ and $g(x,y)=y-\zeta\tilde{g}(x,y)$ for appropriate $\zeta$ (see \cite[Lemma D.1]{Chandak-nonexp} for an explicit conversion). While the same results can be obtained in both frameworks with our averaged noise proof technique, we choose to work with the fixed-point formulation as it is easier to illustrate our proof technique in this formulation. In particular, the intuition behind our auxiliary iterates is much easier to understand in the fixed point formulation (see Lemma 2 and the discussion below).
\end{remark}

The following two assumptions are standard in analysis of SA and are satisfied by the applications we state next. The first assumption states that functions $f(\cdot)$ and $g(\cdot)$ are Lipschitz.
\begin{assumption}\label{assu:Lipschitz}
Functions $f(\cdot,\cdot)$ and $g(\cdot,\cdot)$ are Lipschitz, i.e., 
    \begin{align*}
    &\|f(x_1,y_1)-f(x_2,y_2)\|+\|g(x_1,y_1)-g(x_2,y_2)\|\\
    &\leq L(\|x_1-x_2\|+\|y_1-y_2\|).
\end{align*}
for all $x_1,x_2\in\RR^{d_1}$ and $y_1,y_2\in\RR^{d_2}$. Here $L>0$ is the Lipschitz constant. 
\end{assumption}
Finally, we make the assumption that the noise sequence are martingale difference sequences, and their second moment scales affinely with the squared norm of iterates.
\begin{assumption}\label{assu:Martingale}
Define the family of $\sigma$-fields $\FF_k=\sigma(x_0,y_0,M_i,M'_i, i\leq k)$. Then $\{M_{k+1}\}$ and $\{M'_{k+1}\}$ are martingale difference sequences with respect to $\FF_k$, i.e., $\EE[M_{k+1}\mid\FF_k]=\EE[M'_{k+1}\mid\FF_k]=0.$
Moreover, for all $k\geq 0$,
    $$\EE[\|M_{k+1}\|^2+\|M'_{k+1}\|^2\mid\FF_k]\leq \cc_1 (1+\|x_k\|^2+\|y_k\|^2),$$
for some $\cc_1\geq 0$.
\end{assumption}

\subsection{Applications}\label{sec:applications}
Our framework is fairly general and can incorporate various optimization and RL algorithms. We briefly discuss a few of these settings here.

\subsubsection{SA with Polyak Averaging} This is a two-time-scale algorithm where the slower time-scale is just averaging:
\begin{subequations}
    \begin{align}
                &x_{k+1}=x_k+\alpha_k(F(x_k)-x_k+M_{k+1})\label{Polyak-faster}\\
        &y_{k+1}=y_k+\beta_k(x_k-y_k)=(1-\beta_k)y_k+\beta_kx_k\label{Polyak-slower}.
    \end{align}
\end{subequations}
The aim here is to find the fixed point $\xstar$ for a contractive operator $F(\cdot)$. In presence of Markov noise, $F(x_k)$ is replaced with $F(x_k,W_k)$ in the iteration, and the contractive nature is required for $\overline{F}(\cdot)$, the stationary average of $F(\cdot,w)$. The additional averaging step \eqref{Polyak-slower} is used to improve the statistical efficiency and the rate of convergence \cite{polyak1992acceleration}. This algorithm is compatible with our framework with $\mu=0$. We can obtain a rate of $\mathcal{O}(1/k)$ on $\EE\left[\|y_k-\xstar\|^2\right]$ by choosing $\beta_k=2/(k+1)$. 

Due to the simple structure of the slower time-scale, the  assumption that needs to be carefully verified here is that the \textit{actual} SA iteration \eqref{Polyak-faster} has a mapping contractive under the Euclidean norm. In optimization, iterations of the form $x_{k+1}=x_k+\alpha_k(\tilde{F}(x_k)+M_{k+1})$, where the function $\tilde{F}$ is strongly monotone, are commonly used. Examples include $\tilde{F}(\cdot)=\nabla H(\cdot)$ for a strongly convex and smooth $H(\cdot)$, and $\tilde{F}(\cdot)=\nabla H(\cdot)+\lambda R(\cdot)$ where $H(\cdot)$ is convex and $R(\cdot)$ is a strongly convex regularizer (such as $\ell_2$ regularizer). Similar iterations are also used in RL. Examples include TD(0) with linear approximation for policy evaluation \cite{TD0}, and regularized Q-Learning with linear function approximation \cite{RegQ}. These iterations can then be rewritten in the form of \eqref{Polyak-faster} where the function $F(\cdot)$ is contractive under the Euclidean norm (see Remark \ref{remark-alternate}).

\subsubsection{Stochastic Gradient Descent Ascent Algorithm} Another application of our framework is  minimax optimization for strongly convex-strongly concave functions, i.e., obtaining $\min_y\max_x F(x,y)$ where $F$ is strongly convex in $x$ and strongly concave in $y$. Then consider the following two-time-scale stochastic gradient descent ascent algorithm \cite{minimax}.
\begin{equation*}
    \begin{aligned}
        &x_{k+1}=x_k+\alpha_k(\nabla_xF(x_k,y_k)+M_{k+1})\\
        &y_{k+1}=y_k+\beta_k(-\nabla_yF(x_k,y_k)+M'_{k+1}).
    \end{aligned}
\end{equation*}
The gradient descent (resp.\ ascent) operator is strongly monotone for smooth and strongly convex (resp.\ strongly concave) functions \cite[Theorem 2.1.9]{Nesterov}, and hence the above iterations fit into our framework. This implies that $(x_k,y_k)$ converge to the unique saddle point $(\xstar,\ystar)$ at a rate of $\mathcal{O}(1/k)$. Problems of minimax optimization or saddle point computation frequently arise in two-player zero-sum and matrix games \cite{saddle-games, saddle-games2}.

\subsubsection{Constrained Optimization} Consider  the constrained optimization problem where we wish to maximize $H(x)$ subject to $Ax=b$, and the following two-time-scale Lagrangian optimization method \cite{two-time-lagrangian}.
\begin{equation*}
    \begin{aligned}
        &x_{k+1}=x_k+\alpha_k(\nabla H(x_k)-A^Ty_k+M_{k+1})\\
        &y_{k+1}=y_k+\beta_k(Ax_k-b).
    \end{aligned}
\end{equation*}
Here $y_k$ denotes the Lagrange multiplier and the noise sequence $M_{k+1}$ denotes the noise arising from taking gradient samples at $x_k$. This algorithm is particularly useful in distributed settings, where each node updates some part of the variable $x$ locally, and there is a global linear constraint. It is also applicable in Generalized Nash Equilibrium Problems (GNEP) for linear coupled constraints \cite{GNEP}. Under the assumption that $H$ is strongly concave and smooth, and that the matrix $A$ has full row rank, the above iteration satisfies our assumptions \cite[Lemma 5.6 a)]{Chandak-nonexp}, and the iterates converge at a rate of $\mathcal{O}(1/k)$ to the constrained maxima.
 
\subsubsection{Linear Two-Time-Scale SA} This is a special case of our non-linear framework and hence our results hold true in this setting as well. Consider the problem of finding solutions for the following set of linear equations $$A_{11}x+A_{12}y=b_1\;\;\text{and}\;\;A_{21}x+A_{22}y=b_2.$$
Define the matrix $\Delta=A_{22}-A_{21}A_{11}^{-1}A_{12}$. Then under the assumption that $-A_{11}$ and $-\Delta$ are Hurwitz, i.e., the real part of all their eigenvalues is negative, their exist unique solutions $(\xstar,\ystar)$ to the above linear equations. Moreover, the following two-time-scale iteration fits into our framework of contractive mappings in both time-scales \cite{Shaan}.
\begin{equation*}
    \begin{aligned}
        &x_{k+1}=x_k+\alpha_k(b_1-A_{11}x_k-A_{12}y_k+M_{k+1})\\
        &y_{k+1}=y_k+\beta_k(b_2-A_{21}x_k-A_{22}y_k+M'_{k+1}).
    \end{aligned}
\end{equation*}
The noise sequences here arise due to noisy estimates of the matrix at each time $k$. Iterations of these form have found wide applications in RL, and specifically in off-policy learning where data collected under a behavior policy is used for evaluating a target policy different from the behavior policy. Examples of such algorithms include Temporal Difference with Gradient Correction (TDC) and Gradient Temporal Difference Learning (GTD2) \cite{GTD}.

\section{Main Results}\label{sec:main}
In this section, we present the main results of our paper. The mean square bound depends on the choice of the stepsize sequence. We fix $\beta_k=\mathcal{O}(1/k)$ as the optimal bound for our problem is obtained in this case. Our first result deals with stepsize sequence of the form $\alpha_k=\mathcal{O}(1/k)$. A mean square bound of $\mathcal{O}(1/k)$ is obtained in this case. The second result assumes stepsize sequences of the form $\alpha_k=\mathcal{O}(1/k^{\afrak})$ where $\afrak\in(0.5,1)$. In this case, we obtain a bound of $\mathcal{O}(1/k^\afrak)$.
\subsection{$\alpha_k=\mathcal{O}(1/k)$}
Consider the following assumption on the stepsize sequences.
\begin{assumption}\label{assu:stepsize-1/k}
    $\beta_k$ and $\alpha_k$ are of the form
    $$\beta_k=\frac{\beta}{k+K_1}\;\text{and}\;\alpha_k=\frac{\alpha}{k+K_1},$$
    where $\beta\geq2/(1-\mu)$, $\beta/\alpha\leq C_1$, and $K_1\geq C_2$.
\end{assumption}
The values for constants $C_1$ and $C_2$ are given in Appendix \ref{app:proof-main-1/k}. These constants depends on the system parameters and $C_2$ additionally depends on the choice of $\alpha$ and $\beta$. We wish to make a few remarks about the above assumption. The assumptions that $\beta\geq2/(1-\mu)$ and $\beta/\alpha\leq C_1$ are necessary for our analysis. Since $\beta_k/\alpha_k=\beta/\alpha$ for all $k$, we need $\beta/\alpha\leq C_1$ to specify the separation between the two time-scales. The assumption that $K_1\geq C_2$ is taken to ensure that $\beta_k$ and $\alpha_k$ are sufficiently small for all $k$. This assumption is not necessary, and in absence of this assumption our result will hold for all $k$ greater than some $k_0$. 

Here is our main result.
\begin{theorem}\label{thm:main-1/k}
    Suppose Assumptions \ref{assu:f-contrac}-\ref{assu:Martingale} are satisfied and the stepsize sequences satisfy Assumption \ref{assu:stepsize-1/k}. Then there exist constants $C_3,C_4>0$ such that for all $k\geq 0$,
    $$\EE\left[\|x_k-\xstar(y_k)\|^2+\|y_k-\ystar\|^2\right]\leq \frac{C_3}{k+K_1},$$
    and 
    $$\EE\left[\|x_k-\xstar\|^2\right]\leq \frac{C_4}{k+K_1}.$$
\end{theorem}
Explicit values for constants $C_1,\ldots,C_4$ have been presented along with the theorem's proof in Appendix \ref{app:proof-main-1/k}. An outline of the proof has been given in Section \ref{sec:outline} through a series of lemmas. The following remark discusses whether the result can be achieved by formulating the problem as a single-time-scale SA problem with a joint variable.

\begin{remark}
    Suppose $\gamma=\beta/\alpha$. Then the two iterations in this case can be written as (assuming a noiseless scenario):
    \begin{align*}
        x_{k+1}&=x_k+\alpha_k(f(x_k,y_k)-x_k)\\
        y_{k+1}&=y_k+\alpha_k(\gamma g(x_k,y_k)-\gamma y_k).
    \end{align*}
    Then, by creating a stacked joint variable $\mathbf{W}_k=\bigl[\begin{smallmatrix} x_k \\ y_k \end{smallmatrix}\bigr]$, and by defining the function $H(\mathbf{W}_k)=\bigl[\begin{smallmatrix} f(x_k,y_k)-x_k \\ \gamma g(x_k,y_k)-\gamma y_k \end{smallmatrix}\bigr]$, we can write the above iteration as:
    \begin{align*}
   \mathbf{W}_{k+1}=\mathbf{W}_k+\alpha_kH(\mathbf{W}_k).
    \end{align*}
    To replicate our results using this joint-variable formulation, we would require the function $H(\cdot)$ to be strongly monotone, i.e., the function $H(\cdot)$ should satisfy $\langle H(\mathbf{W})-H(\mathbf{W'}), \mathbf{W}-\mathbf{W'}\rangle\geq c\|\mathbf{W}-\mathbf{W'}\|^2,$ for some $c>0$. This is true in the linear case \cite[Remark 7.3]{Bullo_linear_joint}, but the proof there uses the properties of Hurwitz matrices and cannot be extended to our case. The key difficulty in our setting is controlling the term $\|g(x,y)-g(x',y')\|$ without introducing additional terms of the form $\|x-\xstar(y)\|$, which arise because contractivity of $g$ is available only for $g(\xstar(y),y)$. These extra terms cannot be bounded in a way that yields strong monotonicity of $H(\cdot)$ in general, making this joint-iterate approach inconclusive.

    Nevertheless, certain modifications of our assumptions make the analysis feasible. One such modification involves replacing Assumption \ref{assu:g-contract} with the assumption that the function $g(x,y)$ is $\mu$-contractive in $y$ for all $x$. Then, under mild assumptions on the Lipschitz constant, the mapping $H(\cdot)$ can be shown to be strongly monotone. In this regime, the above update rule can indeed be analyzed to obtain results analogous to ours. This modified formulation includes, as a special case, the stochastic gradient descent–ascent method for strongly convex–strongly concave problems discussed earlier.
\end{remark}

\subsection{$\alpha_k=\mathcal{O}(1/k^\afrak)$ for $\afrak\in(0.5,1)$}
Consider the following assumption on the stepsize sequences.
\begin{assumption}\label{assu:stepsize-1/k^a}
    $\beta_k$ and $\alpha_k$ are of the form
    $$\beta_k=\frac{\beta}{k+K_2}\;\text{and}\;\alpha_k=\frac{\alpha}{(k+K_2)^\afrak},$$
    where $\alpha>0, \afrak\in(0.5,1), \beta\geq2/(1-\mu)$, and $K_2\geq D_1$.
\end{assumption}
The value for constant $D_1$ has been given in Appendix \ref{app:proof-main-1/k^a}. We do not need the assumption $\beta/\alpha\leq C_1$ for such stepsize sequences. Instead, the time-scale separation is specified by $\lim_{k\uparrow\infty}\beta_k/\alpha_k\rightarrow 0$. We take the assumption $K_2\geq D_1$ to ensure that the bound holds for all time $k$. In absence of this assumption, the bound will still hold but only for $k$ greater than some $k_0$. Based on this observation, we want to emphasize that choosing stepsize sequences based on system parameters is not required for the following bound to hold. 

\begin{theorem}\label{thm:main-1/k^a}
        Suppose Assumptions \ref{assu:f-contrac}-\ref{assu:Martingale} are satisfied and the stepsize sequences satisfy Assumption \ref{assu:stepsize-1/k^a}. Then there exist constants $D_2,D_3>0$ such that for all $k\geq 0$,
    $$\EE\left[\|x_k-\xstar(y_k)\|^2+\|y_k-\ystar\|^2\right]\leq \frac{D_2}{(k+K_2)^{\afrak}},$$
    and 
    $$\EE\left[\|x_k-\xstar\|^2\right]\leq \frac{D_3}{(k+K_2)^{\afrak}}.$$
\end{theorem}
Explicit values for constants $D_1,D_2$ and $D_3$ have been presented along with the theorem's proof in Appendix \ref{app:proof-main-1/k^a}. An outline of the proof has been given in Section \ref{sec:outline} through a series of lemmas.

\section{Proof Outline}\label{sec:outline}
In this section, we present an outline of the proofs for Theorems 1 and 2 through a series of lemmas. The proofs for the two theorems differ only in the final few lemmas which are presented in Subsections \ref{subsec:outline-1/k} and \ref{subsec:outline-1/k^a}, respectively. The proofs for the lemmas in this section have been presented in Appendix \ref{app:proof-lemmas} and these lemmas are then used to complete the proofs for Theorem \ref{thm:main-1/k} and Theorem \ref{thm:main-1/k^a} in Appendix \ref{app:proof-main}.

We first present a lemma which states that $\xstar(y_k)$, the `target' points for the faster iteration, move slowly as the iterates $y_k$ move slowly. 
\begin{lemma}\label{lemma:xstar-lip}
     Suppose Assumptions \ref{assu:f-contrac} and \ref{assu:Lipschitz} hold. Then the map $y\mapsto\xstar(y)$ is Lipschitz with parameter $L_0\coloneqq L/(1-\lambda)$, i.e.,
    \begin{equation*}
      \|\xstar(y_1)-\xstar(y_2)\|\leq L_0\|y_1-y_2\|,
    \end{equation*}
    for $y_1,y_2\in\RR^{d_2}$.
\end{lemma}

As stated in the introduction, the key difficulty in obtaining a $\mathcal{O}(1/k)$ mean square bound on non-linear two-time-scale SA is the presence of noise $M'_{k+1}$ in the slower iteration. To counter this issue, we define the averaged noise sequence $U_0=0$ and $U_{k+1}=(1-\beta_k)U_k+\beta_kM'_{k+1}$ for all $k\geq 0$. Note that this averaging is just a proof technique and not a modification to the algorithm. A similar averaged noise variable could have been introduced for the noise $M_k$ as well and would have led to the same bound. We work with noise $M_k$ directly to illustrate how different noise sequences are dealt with.

We also define the modified iterates $z_k=y_k-U_k$ for all $k\geq 0$. In the following lemma, we present how our required bound can be written in terms of the newly defined variables. We also present a rearrangement of \eqref{iter-main} in terms of $z_k$ and finally, we present a bound on the averaged noise sequence.

\begin{lemma}\label{lemma:z_k_and_U_k}
    \begin{enumerate}[label=(\alph*)]
        \item For all $k\geq 0$,
        \begin{align*}
    &\EE\left[\|x_k-\xstar(y_k)\|^2+\|y_k-\ystar\|^2\right]\\
    &\leq 2\EE\left[\|x_k-\xstar(z_k)\|^2+\|z_k-\ystar\|^2\right]\\
    &\;\;+2(1+L_0^2)\EE\left[\|U_k\|^2\right].
\end{align*}
\item The iteration \eqref{iter-main} can be rewritten as:
\begin{equation}\label{iter-changed}
    \begin{aligned}
&x_{k+1}=x_k+\alpha_k (f(x_k,z_k)-x_k+M_{k+1}+d_k)\\
&z_{k+1}=z_k+\beta_k (g(x_k,z_k)-z_k+e_k),
\end{aligned}
\end{equation}
for all $k\geq 0$, where $d_k=f(x_k,y_k)-f(x_k,z_k)$ and $e_k=g(x_k,y_k)-g(x_k,z_k)$. Here, $\|d_k\|^2$ and $\|e_k\|^2$ are both upper bounded by $L^2\|U_k\|^2$.
\item Suppose $\EE\left[1+\|x_i\|^2+\|y_i\|^2\right]\leq \Gamma_1$ for all $ i\leq k-1$ and some $\Gamma_1$, then for $\beta>1$,
$$\EE\left[\|U_m\|^2\right]\leq 2\cc_1\Gamma_1\beta_m, \;\forall m\leq k.$$
    \end{enumerate}
\end{lemma}
The first two parts of the lemma imply that it is sufficient to work with the iteration in \eqref{iter-changed} and obtain a bound on $\EE\left[\|x_k-\xstar(z_k)\|^2+\|z_k-\ystar\|^2\right]$. In the third part of the lemma, we show that $\EE\left[\|U_k\|^2\right]$ decays at a rate of $\beta_k$ if the iterates are assumed to be bounded till time $k-1$.

\textbf{Intuition behind Auxiliary Iterates:}\quad 
The key intuition behind introducing the averaged noise and the auxiliary iterates is to transform the original iteration into one in which the noise sequence has a decaying variance, instead of the constant variance as in the original iteration. This simplifies the analysis significantly and allows for stronger bounds.

When describing the proof technique, we first defined the averaged noise $U_k$, and then the auxiliary iterates $z_k=y_k-U_k$. Equivalently, one may begin by introducing the auxiliary recursion.
\begin{equation}\label{iter-alternate}
    z_{k+1}=z_k+\beta_k(g(x_k,y_k)-z_k).
\end{equation}
Note that this is precisely the iteration in \eqref{iter-changed}. It can then be showed that $y_k=z_k+U_k$, where $U_k$ is the averaged noise sequence. This averaged noise sequence has previously appeared in \cite{Bravo} for analysis of non-expansive SA. 

The use of auxiliary iterations is standard in analysis of stochastic iterations. But these differ from our definition. For example, in \cite{Chandak_conc}, an auxiliary iteration was introduced to obtain concentration bounds for SA. Extending their construction to the two-time-scale setting, the iteration would be:
\begin{equation}\label{iter-wrong}
    z'_{k+1}=z'_k+\beta_k(g(x_k,z'_k)-z'_k).
\end{equation}
As opposed to our definition, this is a noiseless fixed point iteration. Although more intuitive, it does not yield the improvements obtained through our definition \eqref{iter-alternate}, as it lacks the required averaging properties.

We now continue with the proof. Define $\lambda'=1-\lambda$ and $\mu'=1-\mu$. We next present an intermediate recursive bound on $\EE\left[\|x_{m+1}-\xstar(z_{m+1})\|^2\right]$ and $\EE\left[\|z_{m+1}-\ystar\|^2\right]$.

\begin{lemma}\label{lemma:recursive}
    Suppose Assumptions \ref{assu:f-contrac}-\ref{assu:Martingale} are satisfied. Moreover, assume that the stepsize sequences $\alpha_k$ and $\beta_k$ satisfy either Assumption \ref{assu:stepsize-1/k} or Assumption \ref{assu:stepsize-1/k^a}. Then, for all $m\geq 0$,
    \begin{align*}
    &\EE\left[\|x_{m+1}-\xstar(z_{m+1})\|^2\right]\nonumber\\
    &\leq\left(1-\lambda'\alpha_m\right)\EE\left[\|x_m-\xstar(z_m)\|^2\right]+\frac{6L^2\alpha_m}{\lambda'}\EE\left[\|U_m\|^2\right]\nonumber\\
    &\;\;+\frac{\mu'\beta_m}{2}\EE\left[\|z_m-\ystar\|^2\right]+3\alpha_m^2\EE\left[\|M_{m+1}\|^2\right],
\end{align*}
and
\begin{align*}
    &\EE\left[\|z_{m+1}-\ystar\|^2\right]\leq \left(1-\frac{3\mu'\beta_m}{2}\right)\EE\left[\|z_m-\ystar\|^2\right]\nonumber\\
    &+\frac{9\beta_mL^2}{\mu'}\EE\left[\|x_m-\xstar(z_m)\|^2\right]+\frac{9\beta_mL^2}{\mu'}\EE\left[\|U_m\|^2\right].
\end{align*}
\end{lemma}

The proofs for Theorem \ref{thm:main-1/k} and Theorem \ref{thm:main-1/k^a} diverge at this point. Directly adding the two recursive bounds above is sufficient to prove Theorem \ref{thm:main-1/k}. For Theorem \ref{thm:main-1/k^a}, the two recursive bounds are appropriately scaled before adding. This recursion is solved to obtain an intermediate bound which is again substituted in the above lemma to complete the proof.

\subsection{$\alpha_k=\mathcal{O}(1/k)$}\label{subsec:outline-1/k}
For simplicity, define $$S_m=\EE\left[\|x_m-\xstar(z_m)\|^2+\|z_m-\ystar\|^2\right].$$
The following lemma gives a recursive bound on $S_{m+1}$ for $m\leq k-1$ under the assumption that the iterates are bounded till $k-1$. This recursion is then solved in the second part.

\begin{lemma}\label{lemma:recursive-S}
Suppose the setting for Theorem \ref{thm:main-1/k} holds. Suppose $\EE\left[1+\|x_i\|^2+\|y_i\|^2\right]\leq \Gamma_1$ for all $i\leq k-1$ and some constant $\Gamma_1>0$. 
\begin{enumerate}[label=(\alph*)]
    \item For all $m\leq k-1$,
    \begin{align*}
    &S_{m+1}\leq (1-\mu'\beta_m)S_m+4\cc_1\Gamma_1\alpha_m^2.
\end{align*}
\item \begin{align*}
    &S_k\leq S_0\left(\frac{K_1}{k+K_1}\right)+\frac{8\cc_1\Gamma_1}{\mu'\gamma^2}\frac{\beta}{k+K_1}.
\end{align*}
\end{enumerate}
\end{lemma}
The key step for the second part above is to identify that $\beta_k/\alpha_k=\beta/\alpha\eqqcolon \gamma$. Hence the term $4\cc_1\Gamma_1\alpha_m^2$ in part (a) can be replaced with $(4\cc_1\Gamma_1/\gamma^2)\beta_m^2$. This gives us a recursion in a standard form, which is solved to obtain the bound. Note that $z_0=y_0$ and $S_0=\|x_0-\xstar(y_0)\|^2+\|y_0-\ystar\|^2$. 

Next, we use Lemma \ref{lemma:z_k_and_U_k} to get the required bound on $\EE\left[\|x_k-\xstar(y_k)\|^2+\|y_k-\ystar\|^2\right]$. Moreover, we show that the iterates are bounded for all $k$.
\begin{lemma}\label{lemma:final-1/k}
    Suppose the setting for Theorem \ref{thm:main-1/k} holds. Define \begin{align*}
    \Gamma_2&\coloneqq2+4(4L_0^2+2)S_0+8\|\xstar\|^2+4\|\ystar\|^2.
\end{align*} 
If $\EE\left[1+\|x_i\|^2+\|y_i\|^2\right]\leq \Gamma_2$ for all $i\leq k-1$, then
\begin{enumerate}[label=(\alph*)]
    \item $\EE\left[\|x_k-\xstar(y_k)\|^2+\|y_k-\ystar\|^2\right]\leq \frac{C_3}{k+K_1}.$
    \item $\EE\left[1+\|x_k\|^2+\|y_k\|^2\right]\leq \Gamma_2.$
\end{enumerate}
\end{lemma} 
The above lemma shows that if $\EE[1+\|x_i\|^2+\|y_i\|^2]$ is bounded by $\Gamma_2$ for $i\leq k-1$, then $\EE[1+\|x_k\|^2+\|y_k\|^2]$ is also bounded by $\Gamma_2$. By strong induction, this implies that $\EE[1+\|x_k\|^2+\|y_k\|^2]$ is bounded by $\Gamma_2$ for all $k$. Combining this boundedness with the first part of the lemma gives us the required bound on the iterates for all $k$. This has been formally shown in Appendix \ref{app:proof-main-1/k}.

\subsection{$\alpha_k=\mathcal{O}(1/k^\afrak)$ for $\afrak\in(0.5,1)$}\label{subsec:outline-1/k^a}
For simplicity, define 
$$T_m=\EE\left[\frac{18L^2}{\lambda'\mu'}\frac{\beta_m}{\alpha_m}\|x_{m}-\xstar(z_m)\|^2+\|z_m-\ystar\|^2\right].$$
We first give a recursive bound on $T_{m+1}$ for $m\leq k-1$ under the assumption that the iterates are bounded till time $k-1$. This recursion is then solved in the second part.
\begin{lemma}\label{lemma:recursive-T}
        Suppose the setting for Theorem \ref{thm:main-1/k^a} holds. Suppose $\EE\left[1+\|x_i\|^2+\|y_i\|^2\right]\leq \Gamma_1$ for all $i\leq k-1$ and some constant $\Gamma_1>0$.
        \begin{enumerate}[label=(\alph*)]
            \item For all $m\leq k-1$,
            \begin{align*}
    T_{m+1}\leq (1-\mu'\beta_m)T_m+\frac{70L^2\cc_1\Gamma_1}{\lambda'\mu'}\alpha_m\beta_m.
\end{align*}
\item For all $m\leq k$,
$$\EE[\|z_m-\ystar\|^2]\leq S_0\left(\frac{K_2}{m+K_2}\right)+\frac{140L^2\cc_1\Gamma_1}{\lambda'\mu'^2}\alpha_m.$$
        \end{enumerate}
\end{lemma}
Note that solving the recursion gives us a bound on $T_m$. But we only care about the bound on $\EE\left[\|z_m-\ystar\|^2\right]$ as this bound is then substituted back in Lemma \ref{lemma:recursive} to get the required bound on $\EE\left[\|x_k-\xstar(z_k)\|^2\right]$.
\begin{lemma}\label{lemma:bound-x-1/k^a}
    Suppose the setting for Theorem \ref{thm:main-1/k^a} holds. Suppose $\EE\left[1+\|x_i\|^2+\|y_i\|^2\right]\leq \Gamma_1$ for all $i\leq k-1$ and some constant $\Gamma_1>0$. Then,
    \begin{align*}
    &\EE\left[\|x_k-\xstar(z_k)\|^2\right]\leq S_0\left(\frac{K_2}{k+K_2}+\frac{2}{\lambda'}\alpha_k\right)+\frac{8\cc_1\Gamma_1}{\lambda'}\alpha_k.
\end{align*}
\end{lemma}
Having obtained bounds on $\EE\left[\|x_k-\xstar(z_k)\|^2\right]$ and $\EE\left[\|z_k-\ystar\|^2\right]$, we next present an analogue of Lemma \ref{lemma:final-1/k} for the setting of Theorem \ref{thm:main-1/k^a}.
\begin{lemma}\label{lemma:final-1/k^a}
        Suppose the setting for Theorem \ref{thm:main-1/k^a} holds. Define 
$$\Gamma_3\coloneqq2+12(4L_0^2+2)S_0+8\|\xstar\|^2+4\|\ystar\|^2.$$
If $\EE\left[1+\|x_i\|^2+\|y_i\|^2\right]\leq \Gamma_3$ for all $i\leq k-1$, then
\begin{enumerate}[label=(\alph*)]
    \item $\EE\left[\|x_k-\xstar(y_k)\|^2+\|y_k-\ystar\|^2\right]\leq \frac{D_4}{(k+K_2)^\afrak}.$
    \item $\EE\left[1+\|x_k\|^2+\|y_k\|^2\right]\leq \Gamma_3.$
\end{enumerate}
\end{lemma}
Similar to Lemma \ref{lemma:final-1/k}, the above lemma is sufficient to show that the iterates are bounded at all time $k$. This also gives us the required bound on mean square error. This has been formally shown in Appendix \ref{app:proof-main-1/k^a}.

\section{Conclusion and Future Directions}\label{sec:conc}
In this paper, we obtain a finite-time bound of $\mathcal{O}(1/k)$ for non-linear two-time-scale stochastic approximation under contractive assumptions. This bound is obtained by choosing $\alpha_k$ and $\beta_k$ as $\Theta(1/k)$, and does not require any additional assumptions on the iterations. Additionally, a more robust bound of $\mathcal{O}(1/k^{1-\epsilon})$ is also obtained by choosing $\alpha_k=\mathcal{O}(1/k^{1-\epsilon})$ and $\beta_k=\mathcal{O}(1/k)$. 

Under the assumption of local linearity, \cite{Han-linearity} shows that a bound of $\mathcal{O}(1/k)$ is possible for $\EE[\|y_k-\ystar\|^2]$ when $\alpha_k=\mathcal{O}(1/k^\afrak)$ and $\beta_k=\mathcal{O}(1/k)$ for $\afrak<1$. It remains an interesting open problem whether this is possible for non-linear two-time-scale iterations without any additional assumptions. In fact, \cite[Section 4.1]{Han-linearity} provides empirical evidence suggesting that this is not possible in the absence of local linearity, and the best possible rate is $\mathcal{O}(1/k^\afrak)$ which we show in Theorem \ref{thm:main-1/k^a}. We believe our work is a step towards obtaining improved bounds for non-linear two-time-scale SA in the general setting.

An interesting but straightforward future direction is extending our work to two-time-scale SA with Markov noise and arbitrary norm contractions. A bound of $\mathcal{O}(1/k^{2/3})$ was previously obtained in this setting in \cite{Chandak-TTS-CDC}. Our results (Theorem 1 and 2) can be generalized to this broader framework, with only the constants in the bounds requiring modification. The extension amounts to combining our noise averaging technique with the tools used in \cite{Chandak-TTS-CDC} to handle Markov noise and arbitrary norm contractions. For completeness, we briefly outline how these tools integrate with our proof technique (see Subsections IV-A and IV-B of \cite{Chandak-TTS-CDC} for a detailed exposition).

\textbf{Solutions of Poisson Equation:}\quad The method to handle Markov noise relies on the solutions of Poisson equation \cite{poisson}. Using the Poisson equation, the Markov noise can be decomposed into a martingale difference component and a telescoping series. The averaged noise $U_k$ in this case will be the weighted average of both the martingale and Markov noise. The martingale difference component from the Markov noise can directly be combined with the already-existing martingale noise $M_{k+1}$ and $M'_{k+1}$. Therefore the only term which needs to be handled separately is the telescoping series, and it can be shown to decay at the same rate as the other terms.

\textbf{Moreau Envelopes:}\quad The key challenge when working with arbitrary norms is to construct smooth Lyapunov functions which provide negative drift. The squared norm, which acts as the Lyapunov function in this work, is not smooth for norms such as $\ell_\infty$. This challenge can be solved by using Moreau envelopes which give a smooth Lyapunov function with negative drift. This tool is now standard in analysis of SA with arbitrary norms \cite{Zaiwei, Chandak-TTS-CDC}. The remaining proof technique remains the same.

\appendix
\section{Proofs from Section \ref{sec:outline}}\label{app:proof-lemmas}
\subsection{Proof for Lemma \ref{lemma:xstar-lip}}
Note that 
    \begin{align*}
        \|\xstar(y_1)-\xstar(y_2)\|&=\|f(\xstar(y_1),y_1)-f(\xstar(y_2),y_2)\|\\
        &\leq \|f(\xstar(y_1),y_1)-f(\xstar(y_2),y_1)\|+\|f(\xstar(y_2),y_1)-f(\xstar(y_2),y_2)\|\\
        &\leq \lambda\|\xstar(y_1)-\xstar(y_2)\|+L\|y_1-y_2\|. 
    \end{align*}
    Here the first equality follows from the fact that $\xstar(y)$ is a fixed point for $f(\cdot,y)$ and the last inequality follows from contractiveness and Lipschitzness of $f(x,y)$ in $x$ and $y$, respectively. This implies that 
    $(1-\lambda)\|\xstar(y_1)-\xstar(y_2)\|\leq L\|y_1-y_2\|.$
    Hence $\xstar(\cdot)$ is Lipschitz with constant $L_0\coloneqq L/\lambda'$, where $\lambda'=1-\lambda$. 

\subsection{Proof for Lemma \ref{lemma:z_k_and_U_k}}
Recall that $z_k=y_k-U_k$. Then,
\begin{align*}
    \|x_k-\xstar(y_k)\|^2&\leq 2\|x_k-\xstar(z_k)\|^2+2\|\xstar(y_k)-\xstar(z_k)\|^2\\
    &\leq 2\|x_k-\xstar(z_k)\|^2+2L_0^2\|U_k\|^2,
\end{align*}
where the second inequality follows from the Lipschitz nature of $\xstar(\cdot)$. Next,
\begin{align*}
    \|y_k-\ystar\|^2&\leq 2\|z_k-\ystar\|^2+2\|z_k-y_k\|^2\\
    &\leq 2\|z_k-\ystar\|^2+2\|U_k\|^2.
\end{align*}
Taking expectation and adding the above two bounds completes the proof for part (a) of Lemma \ref{lemma:z_k_and_U_k}.

Recall that $U_{k+1}=(1-\beta_k)U_k+\beta_kM'_{k+1}$. Substituting $y_k=z_k+U_k$ in iteration for $y_k$ \eqref{iter-main}, we get  
\begin{align*}
    &z_{k+1}+U_{k+1}=(1-\beta_k)(z_k+U_k)+\beta_k(g(x_k,y_k)+M'_{k+1})\\
    &\implies z_{k+1}=(1-\beta_k)z_k+\beta_kg(x_k,y_k).
\end{align*}
Now, we rewrite the iteration in \eqref{iter-main} as follows.
\begin{equation*}
    \begin{aligned}
&x_{k+1}=x_k+\alpha_k (f(x_k,z_k)-x_k+M_{k+1}+d_k)\\
&z_{k+1}=z_k+\beta_k (g(x_k,z_k)-z_k+e_k),
\end{aligned}
\end{equation*}
where $d_k=f(x_k,y_k)-f(x_k,z_k)$ and $e_k=g(x_k,y_k)-g(x_k,z_k)$. We note that both $\|d_k\|^2$ and $\|e_k\|^2$ are upper bounded by $L^2\|U_k\|^2$. This completes the proof for part (b).

For the third part, note that for $m\geq 1$,
$$U_m=\sum_{i=0}^{m-1}\beta_i\prod_{j=i+1}^{m-1}(1-\beta_j)M'_{i+1}.$$
Suppose $\EE\left[1+\|x_i\|^2+\|y_i\|^2\right]\leq \Gamma_1$ for all $0\leq i\leq k-1$. Then, $\EE\left[\|M'_{i+1}\|^2\right]\leq \cc_1\Gamma_1$ for all $i\leq k-1$ (Assumption \ref{assu:Martingale}).  

Using the property that the terms of a martingale difference sequence are orthogonal, for all $m\leq k$,
\begin{align*}
    \EE\left[\|U_m\|^2\right]&\leq \sum_{i=0}^{m-1}\left(\beta_i\prod_{j=i+1}^{m-1}(1-\beta_j)\right)^2\EE\left[\|M'_{i+1}\|^2\right]\\
    &\leq \cc_1\Gamma_1\sum_{i=0}^{m-1}\beta_i^2\prod_{j=i+1}^{m-1}(1-\beta_j).
\end{align*}
We apply Lemma \ref{lemma:aux} with $\epsilon=\beta^2$, $\qfrak=2$, $\pfrak=1$ and $K=K_1$. Then if $\beta\geq 2$, 
$$\sum_{i=0}^{m-1}\beta_i^2\prod_{j=i+1}^{m-1}(1-\beta_j)\leq 2\beta_m.$$
This implies that $\EE\left[\|U_m\|^2\right]\leq 2\cc_1\Gamma_1\beta_m$ for all $m\leq k$, completing the proof for Lemma \ref{lemma:z_k_and_U_k}.

\subsection{Proof for Lemma \ref{lemma:recursive}}
Throughout this proof, $h(\cdot)$ refers to $g(\xstar(\cdot),\cdot)$. Using Assumption \ref{assu:g-contract}, $h(\cdot)$ is $\mu$-contractive and $\ystar$ is a fixed point for the function $h(\cdot)$. We also define $\lambda'=1-\lambda$ and $\mu'=1-\mu$. All following statements hold true for all $m\geq 0$. 

\subsubsection{Recursive Bound on $\EE\left[\|x_{m+1}-\xstar(z_{m+1})\|^2\right]$}

By adding and subtracting $\xstar(z_m)$ to the iteration for $x_m$ \eqref{iter-changed}, we get
    \begin{subequations}\label{split1}
        \begin{align}
            \|x_{m+1}-\xstar(z_{m+1})\|^2&=\|x_m-\xstar(z_m)+\alpha_m(f(x_m,z_m)-x_m)\|^2\label{split11}\\
            &\;\;+2\Big\langle x_m-\xstar(z_m)+\alpha_m(f(x_m,z_m)-x_m),\nonumber\\
            &\;\;\;\;\;\;\;\;\;\alpha_mM_{m+1}+\alpha_md_m+\xstar(z_m)-\xstar(z_{m+1})\Big\rangle \label{split12}\\
            &\;\;+\|\alpha_mM_{m+1}+\alpha_md_m+\xstar(z_m)-\xstar(z_{m+1})\|^2.\label{split13}
        \end{align}
    \end{subequations}
 We deal with the three terms \eqref{split11}-\eqref{split13} separately as follows.
    \textbf{Term \ref{split11} --- } 
    \begin{align*}
        \|x_m-\xstar(z_m)+\alpha_m(f(x_m,z_m)-x_m)\|^2
        &=\|(1-\alpha_m)(x_m-\xstar(z_m))\\
        &\;\;\;\;\;\;\;\;\;\;\;+\alpha_m(f(x_m,z_m)-f(\xstar(z_m),z_m))\|^2\nonumber\\
        &\leq \big((1-\alpha_m)\|x_m-\xstar(z_m)\|+\alpha_m\lambda\|x_m-\xstar(z_m)\|\big)^2\nonumber\\
        &\leq (1-\lambda'\alpha_m)^2\|x_m-\xstar(z_m)\|^2.
        \end{align*}
 Here, the first equality follows from the fact that $f(\xstar(z_m),z_m)=\xstar(z_m)$, and the first inequality follows from the property that $f(\cdot,y)$ is contractive. For sufficiently large $K_1$ and $K_2$, we have $\lambda'^2\alpha_m^2\leq (\lambda'/4)\alpha_m$ and hence $(1-\lambda'\alpha_m)^2\leq 1-(7\lambda'/4)\alpha_m$. This implies
 \begin{align}\label{split11-simp}
     &\EE\left[\|x_m-\xstar(z_m)+\alpha_m(f(x_m,z_m)-x_m)\|^2\right]\nonumber\\
     &\leq \left(1-\frac{7\lambda'\alpha_m}{4}\right)\EE\left[\|x_m-\xstar(z_m)\|^2\right].
 \end{align}

 \textbf{Term \ref{split12} --- } We first note that $M_{m+1}$ is a martingale difference sequence, which implies that 
 $$\EE\left[\Big\langle x_m-\xstar(z_m)+\alpha_m(f(x_m,z_m)-x_m),M_{m+1}\Big \rangle \mid\FF_m \right]=0. $$
 Now, we note that 
\begin{subequations}\label{split12-split}
\begin{align}
    &2\Big\langle x_m-\xstar(z_m)+\alpha_m(f(x_m,z_m)-x_m),\alpha_md_m+\xstar(z_m)-\xstar(z_{m+1})\Big\rangle\nonumber\\
            &\stackrel{(a)}{\leq} 2(1-\lambda'\alpha_m)\|x_m-\xstar(z_m)\|\|\alpha_md_m+\xstar(z_m)-\xstar(z_{m+1})\|\nonumber\\
            &\stackrel{(b)}{\leq} 2\alpha_m\|x_m-\xstar(z_m)\|\|d_m\|\label{split12-split1}\\
            &\;\;+2L_0\|x_m-\xstar(z_m)\|\|z_{m+1}-z_m\|\label{split12-split2}.
\end{align}
\end{subequations}
For inequality (a), we use Cauchy-Schwarz inequality and the simplification from the bound for term \eqref{split11}. For inequality (b), we use triangle inequality along with Lemma \ref{lemma:xstar-lip}. Additionally we use the fact that $1-\lambda'\alpha_m\leq1$ for all $k$. We deal with the two terms above separately.

\textbf{Term \ref{split12-split1} --- }  Applying AM-GM inequality gives us the following bound.
\begin{align*}
    &2\alpha_m\|x_m-\xstar(z_m)\|\|d_m\|\leq \frac{\lambda'\alpha_m}{4}\|x_m-\xstar(z_m)\|^2+\frac{4\alpha_m}{\lambda'}\|d_m\|^2.
\end{align*}

\textbf{Term \ref{split12-split2} --- } We first recall that $h(\cdot)=g(\xstar(\cdot),\cdot)$.
\begin{align*}
    \|z_{m+1}-z_m\|&= \beta_m\|g(x_m,z_m)-z_m+e_m\|\\
    &\leq \beta_m\big(\|g(x_m,z_m)-h(z_m)\|+\|h(z_m)-h(\ystar)\|+\|h(\ystar)-z_m\|+\|e_m\|\big) \\
    &\leq \beta_m\big(L\|x_m-\xstar(z_m)\|+(1+\mu)\|z_m-\ystar\|+\|e_m\|\big).
\end{align*}
Here, the first inequality follows from triangle inequality after adding and subtracting the terms $h(z_m)$ and $h(\ystar)$. The second inequality follows from the $L$-Lipschitz nature of $g(\cdot,\cdot)$, $\mu$-contractive nature of $h(\cdot)$ and the fact that $h(\ystar)=\ystar$. Now, we use AM-GM inequality twice to get
\begin{align*}
    &2\beta_mL_0(1+\mu)\|x_m-\xstar(z_m)\|\|z_m-\ystar\|\leq \frac{4L_0^2(1+\mu)^2}{\mu'}\beta_m\|x_m-\xstar(z_m)\|^2+\frac{\mu'\beta_m}{4}\|z_m-\ystar\|^2,
\end{align*}
and 
\begin{align*}
    &2\beta_mL_0\|x_m-\xstar(z_m)\|\|e_m\|\leq \beta_m\|x_m-\xstar(z_m)\|^2+L_0^2\beta_m\|e_m\|^2.
\end{align*}
We further use the bound $1+\mu\leq 2$ and combine the above bounds to get the following bound for term \eqref{split12-split2}.
\begin{align*}
    2L_0\|x_m-\xstar(z_m)\|\|z_{m+1}-z_m\|&\leq \left(2LL_0+16L_0^2/\mu'+1\right)\beta_m\|x_m-\xstar(z_m)\|^2\\
    &\;\;+(\mu'/4)\beta_m\|z_m-\ystar\|^2+L_0^2\beta_m\|e_m\|^2.
\end{align*}
Combining this with the bound for \eqref{split12-split1} gives us the following bound.
\begin{align*}
    &2\Big\langle x_m-\xstar(z_m)+\alpha_m(f(x_m,z_m)-x_m),\alpha_md_m+\xstar(z_m)-\xstar(z_{m+1})\Big\rangle\nonumber\\
    &\leq \left(\left(2LL_0+16L_0^2/\mu'+1\right)\beta_m+(\lambda'/4)\alpha_m\right)\|x_m-\xstar(z_m)\|^2\\
    &\;\;+(\mu'/4)\beta_m\|z_m-\ystar\|^2+(4\alpha_m/\lambda')\|d_m\|^2+L_0^2\beta_m\|e_m\|^2.
\end{align*}
For sufficiently small $\beta/\alpha$ under the setting of Theorem \ref{thm:main-1/k} and for sufficiently large $K_2$ under the setting of Theorem \ref{thm:main-1/k^a}, we have $(2LL_0+16L_0^2/\mu'+1)\beta_m\leq (\lambda'/4)\alpha_m$. Finally this gives us the following bound on the expectation of \eqref{split12}.
\begin{align}\label{split12-simp}
    &2\EE\Big[\Big\langle x_m-\xstar(z_m)+\alpha_m(f(x_m,z_m)-x_m),\nonumber\\
    &\;\;\;\;\;\;\;\;\;\;\;\;\;\;\;\;\;\;\alpha_mM_{m+1}+\alpha_md_m+\xstar(z_m)-\xstar(z_{m+1})\Big\rangle\Big] \nonumber\\
    &\leq \frac{\lambda'\alpha_m}{2}\EE\left[\|x_m-\xstar(z_m)\|^2\right]+\frac{\mu'\beta_m}{4}\EE\left[\|z_m-\ystar\|^2\right]\nonumber\\
    &\;\;+\frac{4\alpha_m}{\lambda'}\EE[\|d_m\|^2]+L_0^2\beta_m\EE[\|e_m\|^2].
\end{align}

\textbf{Term \ref{split13} --- } Note that
        \begin{align*}
            &\|\alpha_mM_{m+1}+\alpha_md_m+\xstar(z_m)-\xstar(z_{m+1})\|^2\leq 3\alpha_m^2\|M_{m+1}\|^2+3\alpha_m^2\|d_m\|^2+3\|\xstar(z_{m+1})-\xstar(z_m)\|^2\nonumber.
        \end{align*}
        Similar to the simplification for the term \eqref{split12-split2},
        \begin{align*}
            3\|\xstar(z_{m+1})-\xstar(z_m)\|^2&\leq 3L_0^2\beta_m^2\|g(x_m,z_m)-z_m+e_m\|^2\\
            &\leq 9L_0^2\beta_m^2\left(L^2\|x_m-\xstar(z_m)\|^2+4\|z_m-\ystar\|^2+\|e_m\|^2\right).
        \end{align*}
        For sufficiently small $\beta/\alpha$ and sufficiently large $K_1$ in the setting of Theorem \ref{thm:main-1/k} and sufficiently large $K_2$ in the setting of Theorem \ref{thm:main-1/k^a}, we have
        \begin{itemize}
            \item $9L_0^2L^2\beta_m^2\leq (\lambda'/4)\beta_m\leq (\lambda'/4)\alpha_m$,
            \item $36L_0^2\beta_m^2\leq (\mu'/4)\beta_m$,
            \item $3\alpha_m^2\leq \alpha_m/\lambda'$,
            \item $9L_0^2\beta_m^2\leq L_0^2\beta_m$.
        \end{itemize}
        This gives us the following bound on the expectation of term \eqref{split13}.
        \begin{align}\label{split13-simp}
            &\EE\left[\|\alpha_mM_{m+1}+\alpha_md_m+\xstar(z_m)-\xstar(z_{m+1})\|^2\right]\nonumber\\
            &\leq \frac{\lambda'\alpha_m}{4}\EE\left[\|x_m-\xstar(z_m)\|^2\right]+\frac{\mu'\beta_m}{4}\EE\left[\|z_m-\ystar\|^2\right]\nonumber
            \\
            &\;\;+3\alpha_m^2\EE\left[\|M_{m+1}\|^2\right]+\frac{\alpha_m}{\lambda'}\EE\left[\|d_m\|^2\right]+L_0^2\beta_m\EE\left[\|e_m\|^2\right].
        \end{align}

Combining the bounds \eqref{split11-simp}, \eqref{split12-simp}, \eqref{split13-simp}, we get the following bound on $\EE[\|x_{m+1}-\xstar(z_{m+1})\|^2]$.
\begin{align*}
    \EE\left[\|x_{m+1}-\xstar(z_{m+1})\|^2\right]&\leq 3\alpha_m^2\EE\left[\|M_{m+1}\|^2\right]\nonumber\\
    &\;\;+\left(1-\lambda'\alpha_m\right)\EE\left[\|x_m-\xstar(z_m)\|^2\right]+\frac{5\alpha_m}{\lambda'}\EE\left[\|d_m\|^2\right]\nonumber\\
    &\;\;+\frac{\mu'\beta_m}{2}\EE\left[\|z_m-\ystar\|^2\right]+2L_0^2\beta_m\EE\left[\|e_m\|^2\right].
\end{align*}
Now, recall that $\|d_m\|^2$ and $\|e_m\|^2$ are upper bounded by $L^2\|U_m\|^2$. For sufficiently small $\beta/\alpha$ in the setting of Theorem \ref{thm:main-1/k} and sufficiently large $K_2$ in the setting of Theorem \ref{thm:main-1/k^a}, we have $2L_0^2\beta_m\leq \alpha_m/\lambda'$. This gives us the required intermediate recursive bound on $\EE[\|x_{m+1}-\xstar(z_{m+1})\|^2]$, completing the proof for part a). 
\begin{align}\label{inter-x}
    &\EE\left[\|x_{m+1}-\xstar(z_{m+1})\|^2\right]\nonumber\\
    &\leq\left(1-\lambda'\alpha_m\right)\EE\left[\|x_m-\xstar(z_m)\|^2\right]+\frac{6L^2\alpha_m}{\lambda'}\EE\left[\|U_m\|^2\right]+\frac{\mu'\beta_m}{2}\EE\left[\|z_m-\ystar\|^2\right]+3\alpha_m^2\EE\left[\|M_{m+1}\|^2\right].
\end{align}
\subsubsection{Recursive Bound on $\EE\left[\|z_{m+1}-\ystar\|^2\right]$}
Recall that $h(\cdot)=g(\xstar(\cdot),\cdot)$. Then, 
\begin{align*}
    \|z_{m+1}-\ystar\|&=\|(1-\beta_m)(z_m-\ystar)+\beta_m(g(x_m,z_m)-\ystar+e_m)\|\\
    &\leq(1-\beta_m)\|(z_m-\ystar)\|+\beta_m\|h(z_m)-h(\ystar)\|+\beta_m\|g(x_m,z_m)-h(z_m)\|+\beta_m\|e_m\|\\
    &\leq (1-\mu'\beta_m)\|z_m-\ystar\|+\beta_m L\left(\|x_m-\xstar(z_m)\|+\|U_m\|\right).
\end{align*}
Here the first inequality follows from triangle inequality and the fact that $\ystar$ is a fixed point for $h(\cdot)$. The second inequality follows from $\mu$-contractive nature of $h(\cdot)$ and $L$-Lipschitz nature of $g(\cdot)$. Now, squaring both sides, we get
\begin{align*}
    \|z_{m+1}-\ystar\|^2&\leq (1-\mu'\beta_m)^2\|z_m-\ystar\|^2+L^2\beta_m^2(\|x_m-\xstar(z_m)\|+\|U_m\|)^2\\
    &\;\;+ 2(1-\mu'\beta_m)L\beta_m\|z_m-\ystar\|(\|x_m-\xstar(z_m)\|+\|U_m\|).
\end{align*}
For the second term, note that 
$$(\|x_m-\xstar(z_m)\|+\|U_m\|)^2\leq 2\|x_m-\xstar(z_m)\|^2+2\|U_m\|^2.$$
For the third term, we first note that $1-\mu'\beta_m\leq 1$. Then,
\begin{align*}
    &2\beta_mL\|z_m-\ystar\|(\|x_m-\xstar(z_m)\|+\|U_m\|)\\
    &\leq \frac{\mu'\beta_m}{4}\|z_m-\ystar\|^2+\frac{4\beta_mL^2}{\mu'}(\|x_m-\xstar(z_m)\|+\|U_m\|)^2\\
    &\leq \frac{\mu'\beta_m}{4}\|z_m-\ystar\|^2+\frac{8\beta_mL^2}{\mu'}\left(\|x_m-\xstar(z_m)\|^2+\|U_m\|^2\right).
\end{align*}

For sufficiently large $K_1$ and $K_2$, we have
\begin{itemize}
    \item $\mu'^2\beta_m^2\leq \mu'\beta_m/4\implies (1-\mu'\beta_m)^2\leq 1-(7\mu'/4)\beta_m$,
    \item $2\beta_m^2\leq \beta_m/\mu'$,
\end{itemize}
This gives us the following intermediate recursive bound on $\EE[\|z_{m+1}-\ystar\|^2]$.
\begin{align}\label{inter-z}
    &\EE\left[\|z_{m+1}-\ystar\|^2\right]\leq \left(1-\frac{3\mu'\beta_m}{2}\right)\EE\left[\|z_m-\ystar\|^2\right]+\frac{9\beta_mL^2}{\mu'}\EE\left[\|x_m-\xstar(z_m)\|^2\right]+\frac{9\beta_mL^2}{\mu'}\EE\left[\|U_m\|^2\right].
\end{align}
\subsection{Proof for Lemma \ref{lemma:recursive-S}}
\subsubsection{(a) Recursive Bound on $S_{m+1}$}
Adding the two intermediate bounds \eqref{inter-x} and \eqref{inter-z}, we get
\begin{align*}
    S_{m+1}&\leq \left(1-\lambda'\alpha_m+\frac{9\beta_mL^2}{\mu'}\right)\EE\left[\|x_m-\xstar(z_m)\|^2\right]\\
    &\;\;+(1-\mu'\beta_m)\EE\left[\|z_m-\ystar\|^2\right]+3\alpha_m^2\EE\left[\|M_{m+1}\|^2\right]\\
    &\;\;+\left(\frac{6L^2\alpha_m}{\lambda'}+\frac{9\beta_mL^2}{\mu'}\right)\EE\left[\|U_m\|^2\right].
\end{align*}
For sufficiently small $\beta/\alpha$, we have 
\begin{itemize}
    \item $1-\lambda'\alpha_m+(9\beta_mL^2/\mu')\leq 1-\mu'\beta_m$,
    \item $9\beta_m/\mu'\leq \alpha_m/\lambda'$.
\end{itemize}
This gives us
\begin{align*}
    S_{m+1}&\leq (1-\mu'\beta_m)S_m\nonumber\\
    &\;\;+3\alpha_m^2\EE\left[\|M_{m+1}\|^2\right]+\frac{7L^2\alpha_m}{\lambda'}\EE\left[\|U_m\|^2\right].
\end{align*}

Now, suppose that $\EE\left[1+\|x_i\|^2+\|y_i\|^2\right]\leq \Gamma_1$ for all $0\leq i\leq k-1$. Then $\EE\left[\|M_{m+1}\|^2\right]\leq \cc_1\Gamma_1$, and $\EE\left[\|U_m\|^2\right]\leq 2\cc_1\Gamma_1\beta_m$ for all $m\leq k-1$.
This gives us
\begin{align*}
    &S_{m+1}\leq (1-\mu'\beta_m)S_m+3\alpha_m^2\cc_1\Gamma_1+\frac{14L^2\cc_1}{\lambda'}\Gamma_1\alpha_m\beta_m,
\end{align*}
for all $m\leq k-1$.
For sufficiently small $\beta/\alpha$, we have $(14L^2/\lambda')\beta_m\leq \alpha_m$. This implies that for all $m\leq k-1$,
\begin{align}\label{inter-x-z}
    &S_{m+1}\leq (1-\mu'\beta_m)S_m+4\cc_1\Gamma_1\alpha_m^2.
\end{align}

\subsubsection{(b) Bound on $S_k$}
For the case of $\alpha_k=\alpha/(k+K_1)$ and $\beta_m=\beta/(k+K_1)$, define $\gamma=\beta/\alpha$. Then $\alpha_m^2=\beta_m^2/\gamma^2$. Then,
\begin{align*}
    &S_{m+1}\leq (1-\mu'\beta_m)S_m+(4\cc_1\Gamma_1/\gamma^2)\beta_m^2.
\end{align*}
Recall that $z_0=y_0$. Iterating from $i=0$ to $k$, we get
\begin{align*}
    S_k&\leq S_0\prod_{j=0}^{k-1}(1-\mu'\beta_j)\\
    &\;\;+(4\cc_1\Gamma_1/\gamma^2)\sum_{i=0}^{k-1}\beta_i^2\sum_{j=i+1}^{k-1}(1-\mu'\beta_j).
\end{align*}
Using Corollary 2.1.2 from \cite{Zaiwei}, for $\beta> 2/\mu'$, we have
$$\prod_{j=0}^{k-1}(1-\mu'\beta_j)\leq \frac{K_1}{k+K_1}.$$
Using Lemma \ref{lemma:aux}, with $\epsilon=\beta^2$, $\qfrak=2$, $\pfrak=\mu'$, and $K=K_1$ we have for $\beta\geq 2/\mu'$
$$\sum_{i=0}^{k-1}\beta_i^2\sum_{j=i+1}^{k-1}(1-\mu'\beta_j)\leq \frac{2}{\mu'}\beta_k.$$
This finally gives us the following bound.
\begin{align}\label{1/k-almost-done}
    &S_k\leq S_0\left(\frac{K_1}{k+K_1}\right)+\frac{8\cc_1\Gamma_1}{\mu'\gamma^2}\frac{\beta}{k+K_1}.
\end{align}

\subsection{Proof for Lemma \ref{lemma:final-1/k}}
We first define \begin{align*}
    \Gamma_2\coloneqq&2+4(4L_0^2+2)S_0+8\|\xstar\|^2+4\|\ystar\|^2.
\end{align*} 
If $\EE\left[1+\|x_i\|^2+\|y_i\|^2\right]\leq \Gamma_2$ for all $0\leq i\leq k-1$, then substituting $\Gamma_2$ in place of $\Gamma_1$ in Lemma \ref{lemma:z_k_and_U_k} and bound \eqref{1/k-almost-done}, we get
\begin{align*}
    &\EE\left[\|x_k-\xstar(y_k)\|^2+\|y_k-\ystar\|^2\right]\\
    &\leq 2S_0\left(\frac{K_1}{k+K_1}\right)+\frac{16\cc_1\Gamma_2}{\mu'\gamma^2}\frac{\beta}{k+K_1}+2(1+L_0^2)2\cc_1\Gamma_2\frac{\beta}{k+K_1}\\
    &\eqqcolon \frac{C_3}{k+K_1},
\end{align*}
for $C_3$ as defined in Appendix \ref{app:proof-main-1/k}. This completes the proof for the first part of Lemma \ref{lemma:final-1/k}.

We next want to show that $\EE\left[1+\|x_k\|^2+\|y_k\|^2\right]$ is bounded by $\Gamma_2$. For this, note that 
\begin{align*}
    \|x_k\|^2&\leq 2\|x_k-\xstar(y_k)\|^2+2\|\xstar(y_k)\|^2\\
    &\leq 2\|x_k-\xstar(y_k)\|^2+4\|\xstar(y_k)-\xstar\|^2+4\|\xstar\|^2\\
    &\leq 2\|x_k-\xstar(y_k)\|^2+4L_0^2\|y_k-\ystar\|^2+4\|\xstar\|^2
\end{align*}
Similarly,
\begin{align*}
    \|y_k\|^2\leq 2\|y_k-\ystar\|^2+2\|\ystar\|^2.
\end{align*}
Now,
\begin{align*}
    &\EE\left[1+\|x_k\|^2+\|y_k\|^2\right]\\
    &\leq 1+(4L_0^2+2)\EE\left[\|x_k-\xstar(y_k)\|^2+\|y_k-\ystar\|^2\right]+4\|\xstar\|^2+2\|\ystar\|^2\\
    &\leq 1+2(4L_0^2+2)S_0+4\|\xstar\|^2+2\|\ystar\|^2+(4L_0^2+2)\left(\frac{16\cc_1\Gamma_2}{\mu'\gamma^2}\beta_k+2(1+L_0^2)2\cc_1\Gamma_2\beta_k\right).
\end{align*}
For the second inequality here, we use the fact that $K_1/(k+K_1)\leq 1$ for all $k$. For sufficiently large $K_1$, $$(4L_0^2+2)\left(\frac{16\cc_1}{\mu'\gamma^2}+2(1+L_0^2)2\cc_1\right)\beta_k\leq 0.5.$$ 
This implies that
\begin{align*}
    \EE\left[1+\|x_k\|^2+\|y_k\|^2\right]&\leq 1+2(4L_0^2+2)S_0+0.5\Gamma_2+4\|\xstar\|^2+2\|\ystar\|^2\\
    &\leq 2+4(4L_0^2+2)S_0+8\|\xstar\|^2+4\|\ystar\|^2\\
    &=\Gamma_2.
\end{align*}
This completes the proof for Lemma \ref{lemma:final-1/k}.

\subsection{Proof for Lemma \ref{lemma:recursive-T}} 
\subsubsection{Recursive Bound on $T_{m+1}$} Under the assumption, we first note that $\EE\left[\|U_m\|^2\right]\leq 2\cc_1\Gamma_1\beta_m$ and $\EE\left[\|M_{m+1}\|^2\right]\leq \cc_1\Gamma_1$. Substituting these bounds in Lemma \ref{lemma:recursive} gives us
\begin{align*}
    \EE\left[\|x_{m+1}-\xstar(z_{m+1})\|^2\right]&\leq (1-\lambda'\alpha_m)\EE\left[\|x_m-\xstar(z_m)\|^2\right]+\frac{12L^2\cc_1\Gamma_1}{\lambda'}\alpha_m\beta_m\\
    &\;\;+\frac{\mu'\beta_m}{2}\EE\left[\|z_m-\ystar\|^2\right]+3\cc_1\Gamma_1\alpha_m^2.
\end{align*}
We first note that $\beta_{m+1}/\alpha_{m+1}\leq \beta_m/\alpha_m$. Furthermore for sufficiently large $K_2$, we have $\frac{18L^2}{\lambda'\mu'}\frac{\beta_m}{\alpha_m}\leq 1$ for all $m$. Then,
\begin{align*}
    \EE\left[\frac{18L^2}{\lambda'\mu'}\frac{\beta_{m+1}}{\alpha_{m+1}}\|x_{m+1}-\xstar(z_{m+1})\|^2\right]&\leq\EE\left[\frac{18L^2}{\lambda'\mu'}\frac{\beta_m}{\alpha_m}\|x_{m+1}-\xstar(z_{m+1})\|^2\right]\\
    &\leq \frac{18L^2}{\lambda'\mu'}\frac{\beta_m}{\alpha_m}(1-\lambda'\alpha_m)\EE\left[\|x_m-\xstar(z_m)\|^2\right]\\
    &\;\;+\frac{\mu'\beta_m}{2}\EE\left[\|z_m-\ystar\|^2\right]+\frac{66L^2\cc_1\Gamma_1}{\lambda'\mu'}\alpha_m\beta_m.
\end{align*}
For the last term here, we use the fact that 
\begin{align*}
    \frac{18L^2}{\lambda'\mu'}\frac{\beta_m}{\alpha_m}\frac{12L^2\cc_1\Gamma_1}{\lambda'}\alpha_m\beta_m&\leq \frac{12L^2\cc_1\Gamma_1}{\lambda'}\alpha_m\beta_m\leq \frac{12L^2\cc_1\Gamma_1}{\lambda'\mu'}\alpha_m\beta_m
\end{align*}
Adding the above scaled bound with the bound for $\EE\left[\|z_{m+1}-\ystar\|^2\right]$ gives us
\begin{align*}
    T_{m+1}&\leq \frac{18L^2}{\lambda'\mu'}\frac{\beta_m}{\alpha_m}\left(1-\frac{\lambda'\alpha_m}{2}\right)\EE\left[\|x_m-\xstar(z_m)\|^2\right]\\
    &\;\;+(1-\mu'\beta_m)\EE\left[\|z_m-\ystar\|^2\right]+\frac{18L^2\cc_1\Gamma_1}{\mu'}\beta_m^2+\frac{66L^2\cc_1\Gamma_1}{\lambda'\mu'}\alpha_m\beta_m. 
\end{align*}
For sufficiently large $K_2$, we have
\begin{itemize}
    \item $1-\lambda'\alpha_m/2\leq 1-\mu'\beta_m$
    \item $18\beta_m\leq 4\alpha_m/\lambda'$.
\end{itemize}
This gives us the following recursive bound.
\begin{align*}
    T_{m+1}\leq (1-\mu'\beta_m)T_m+\frac{70L^2\cc_1\Gamma_1}{\lambda'\mu'}\alpha_m\beta_m.
\end{align*}

\subsubsection{Bound on $T_{m}$}
For all $m\leq k$, note that 
\begin{align*}
    T_m&\leq T_0\prod_{j=0}^{m-1}(1-\mu'\beta_j)+\frac{70L^2\cc_1\Gamma_1}{\lambda'\mu'}\sum_{i=0}^{m-1}\alpha_i\beta_i\prod_{j=i+1}^{m-1}(1-\mu'\beta_j).
\end{align*}
Note that $T_0\leq S_0$ and the first term is bounded in the same way as the proof for Lemma \ref{lemma:recursive-S}. For the second term, we use Lemma \ref{lemma:aux} with $\epsilon=\alpha\beta, \pfrak=\mu',\qfrak=1+\afrak$ and $K=K_2$. For $\beta\geq 2\afrak/\mu'$, $$\sum_{i=0}^{m-1}\alpha_i\beta_i\prod_{j=i+1}^{m-1}(1-\mu'\beta_j)\leq \frac{2}{\mu'}\alpha_m.$$
This gives us the bound on $T_m$. Note that $\EE\left[\|z_m-\ystar\|^2\right]\leq T_m$. This completes the proof for Lemma \ref{lemma:recursive-T}.

\subsection{Proof for Lemma \ref{lemma:bound-x-1/k^a}}
Using the bound from Lemma \ref{lemma:recursive} and substituting the bound on $\EE\left[\|z_m-\ystar\|^2\right]$ from Lemma \ref{lemma:recursive-T}, we get for $m\leq k-1$
\begin{align*}
    \EE\left[\|x_{m+1}-\xstar(z_{m+1})\|^2\right]&\leq (1-\lambda'\alpha_m)\EE\left[\|x_m-\xstar(z_m)\|^2\right]+\frac{12L^2\cc_1\Gamma_1}{\lambda'}\alpha_m\beta_m\\
    &\;\;+\frac{\mu'\beta_m}{2}\left(S_0\left(\frac{K_2}{m+K_2}\right)+\frac{140L^2\cc_1\Gamma_1}{\lambda'\mu'^2}\alpha_m\right)+3\cc_1\Gamma_1\alpha_m^2.
\end{align*}
Here we have also used the bounds on $\EE[\|U_m\|^2]$ and $\EE[\|M_{m+1}\|^2]$ which follow from our assumption that $\EE[1+\|x_i\|^2+\|y_i\|^2]\leq \Gamma_1$ for all $i\leq k-1$. Now,
$$\frac{\mu'\beta K_2}{2(m+K_2)^2}\leq \frac{\alpha^2}{(m+K_2)^{2\afrak}},$$
if $K_2^{1-\afrak}\geq \frac{\mu'\beta}{2\alpha^2}$. For sufficiently large $K_2$, we also have
$$\frac{12L^2}{\lambda'}\beta_m+\frac{70L^2}{\lambda'\mu'}\beta_m\leq \alpha_m.$$
This gives us the bound
\begin{align*}
    &\EE\left[\|x_{m+1}-\xstar(z_{m+1})\|^2\right]\leq (1-\lambda'\alpha_m)\EE\left[\|x_m-\xstar(z_m)\|^2\right]+(S_0+4\cc_1\Gamma_1)\alpha_m^2.
\end{align*}
Iterating the recursion from $m=0$ to $k-1$ gives us
\begin{align*}
    \EE\left[\|x_k-\xstar(z_k)\|^2\right]&\leq \EE\left[\|x_0-\xstar(y_0)\|^2\right]\prod_{j=0}^{k-1}(1-\lambda'\alpha_j)\\
    &+(S_0+4\cc_1\Gamma_1)\sum_{i=0}^{k-1}\alpha_i^2\prod_{j=i+1}^{m-1}(1-\lambda'\alpha_j).
\end{align*}
For the first term, we first note that $\EE\left[\|x_0-\xstar(y_0)\|^2\right]\leq S_0$. Next, for appropriate $D_3$, $1-\lambda'\alpha_j\leq 1-\mu'\beta_j$. Hence, 
$$\EE\left[\|x_0-\xstar(y_0)\|^2\right]\prod_{j=0}^{k-1}(1-\lambda'\alpha_j)\leq S_0\frac{K_2}{k+K_2}.$$
We use Corollary 2.1.2 from \cite{Zaiwei} for the second term to get
$$\sum_{i=0}^{k-1}\alpha_i^2\prod_{j=i+1}^{m-1}(1-\lambda'\alpha_j)\leq \frac{2}{\lambda'}\alpha_k,$$
for $K_2\geq \left(\frac{2\afrak}{\lambda'\alpha}\right)^{1/1-\afrak}$. This gives us the following bound.
\begin{align*}
    &\EE\left[\|x_k-\xstar(z_k)\|^2\right]\leq S_0\left(\frac{K_2}{k+K_2}+\frac{2}{\lambda'}\alpha_k\right)+\frac{8\cc_1\Gamma_1}{\lambda'}\alpha_k.
\end{align*}

\subsection{Proof for Lemma \ref{lemma:final-1/k^a}}
We first define 
$$\Gamma_3\coloneqq2+12(4L_0^2+2)S_0+8\|\xstar\|^2+4\|\ystar\|^2.$$
Substituting $\Gamma_3$ in place of $\Gamma_1$ in Lemma \ref{lemma:z_k_and_U_k} and the bounds from Lemma \ref{lemma:recursive-T} and Lemma \ref{lemma:bound-x-1/k^a}, we get
\begin{align*}
    &\EE\left[\|x_k-\xstar(y_k)\|^2+\|y_k-\ystar\|^2\right] \\
    &\leq 2\EE\left[\|x_k-\xstar(z_k)\|^2\!+\!\|z_k-\ystar\|^2\right]+2(1+L_0^2)\EE\left[\|U_k\|^2\right]\\
    &\leq 2S_0\left(\frac{2K_2}{k+K_2}+\frac{2}{\lambda'}\alpha_k\right)+\left(\frac{280L^2\cc_1}{\lambda'\mu'^2}+\frac{16\cc_1}{\lambda'}\right)\Gamma_3\alpha_k+2(1+L_0^2)2\cc_1\Gamma_3\beta_k\\
    &\leq 2S_0\left(\frac{2K_2}{(k+K_2)^{\afrak}}+\frac{2}{\lambda'}\alpha_k\right)+\left(\frac{280L^2\cc_1}{\lambda'\mu'^2}+\frac{16\cc_1}{\lambda'}\right)\Gamma_3\alpha_k+2(1+L_0^2)2\cc_1\Gamma_3\alpha_k
\end{align*}
The last inequality here follows from the fact that $1/(k+K_2)\leq 1/(k+K_2)^\afrak$ and $\beta_k\leq \alpha_k$. This gives us
$$\EE\left[\|x_k-\xstar(y_k)\|^2+\|y_k-\ystar\|^2\right]\leq \frac{D_2}{(k+K_2)^\afrak},$$
where $D_2$ is defined in Appendix \ref{app:proof-main-1/k^a}. This completes the proof for the first part of Lemma \ref{lemma:final-1/k^a}. 

Similar to Lemma \ref{lemma:final-1/k}, we note that 
\begin{align*}
    &\EE\left[1+\|x_k\|^2+\|y_k\|^2\right]\\
    &\leq 1+(4L_0^2+2)\EE\left[\|x_k-\xstar(y_k)\|^2+\|y_k-\ystar\|^2\right]+4\|\xstar\|^2+2\|\ystar\|^2\\
    &\leq 1+(4L_0^2+2)\times 6S_0+4\|\xstar\|^2+2\|\ystar\|^2+(4L_0^2+2)\left(\frac{280L^2\cc_1}{\lambda'\mu'^2}+\frac{16\cc_1}{\lambda'}+2(1+L_0^2)2\cc_1\right)\Gamma_3\alpha_k.
\end{align*}
For the second inequality here, we use the fact that $K_2/(k+K_2)\leq 1$ and the assumption that $2\alpha_k/\lambda'\leq 1$. Now for sufficiently large $K_2$,
$$(4L_0^2+2)\left(\frac{280L^2\cc_1}{\lambda'\mu'^2}+\frac{16\cc_1}{\lambda'}+2(1+L_0^2)2\cc_1\right)\alpha_k\leq 0.5,$$
for all $k$. Then, 
\begin{align*}
    \EE\left[1+\|x_k\|^2+\|y_k\|^2\right]&\leq 1+6(4L_0^2+2)S_0+4\|\xstar\|^2+2\|\ystar\|^2+0.5\Gamma_3\\
    &\leq 2+12(4L_0^2+2)S_0+8\|\xstar\|^2+8\|\ystar\|^2\\
    &=\Gamma_3.
\end{align*}
This completes the proof for Lemma \ref{lemma:final-1/k^a}.

\section{Proofs for Theorem \ref{thm:main-1/k} and Theorem \ref{thm:main-1/k^a}}\label{app:proof-main}
\subsection{Proof for Theorem \ref{thm:main-1/k}}\label{app:proof-main-1/k}
Lemma \ref{lemma:final-1/k} shows that if $\EE\left[1+\|x_i\|^2+\|y_i\|^2\right]\leq \Gamma_2$ for all $i\leq k-1$, then $\EE\left[1+\|x_k\|^2+\|y_k\|^2\right]\leq \Gamma_2$. For the base case of $k=0$, note that
\begin{align*}
    \EE\left[1+\|x_0\|^2+\|y_0\|^2\right]&\leq 1+(4L_0^2+2)S_0+4\|\xstar\|^2+2\|\ystar\|^2\\
    &\leq \Gamma_2.
\end{align*}
Hence using the law of strong induction, $\EE\left[1+\|x_k\|^2+\|y_k\|^2\right]\leq \Gamma_2$ for all $k\geq 0$. This implies that 
$$\EE\left[\|x_k-\xstar(y_k)\|^2+\|y_k-\ystar\|^2\right]\leq \frac{C_3}{k+K_1},$$
for all $k$. Now,
\begin{align*}
    \EE\left[\|x_k-\xstar\|^2\right]&\leq 2\EE\left[\|x_k-\xstar(y_k)\|^2\right]+2L_0^2\EE\left[\|y_k-\ystar\|^2\right]\\
    &\leq 2(L_0^2+1)\frac{C_3}{k+K_1}\\
    &\eqqcolon \frac{C_4}{k+K_1}.
\end{align*}

\subsubsection{Values of Constants in Assumption \ref{assu:stepsize-1/k} and Theorem \ref{thm:main-1/k}}
We assume $\beta/\alpha\leq C_1$, where 
$$C_1=\frac{\lambda'\mu'}{8LL_0+64L_0^2+4+14L^2},$$
and $K_1\geq C_2$, where 
\begin{align*}
    &C_2=4\alpha+36L_0^2\beta\left(\frac{L^2}{\lambda'}+\frac{4}{\mu'}+1\right)+16\cc_1\beta(2L_0^2+1)\left(\frac{4}{\mu'\gamma^2}+1+L_0^2\right).
\end{align*}
The constants $C_3$ and $C_4$ in the bound are
$$C_3=2S_0K_1+\frac{16\cc_1\Gamma_2\beta}{\mu'\gamma^2}+4\cc_1\Gamma_2\beta(1+L_0^2),$$
$$C_4=2(L_0^2+1)C_3.$$
Here $\Gamma_2=2+4(4L_0^2+2)S_0+8\|\xstar\|^2+4\|\ystar\|^2$.

\subsection{Proof for Theorem \ref{thm:main-1/k^a}}\label{app:proof-main-1/k^a}
Lemma \ref{lemma:final-1/k^a} shows that if $\EE\left[1+\|x_i\|^2+\|y_i\|^2\right]\leq \Gamma_3$ for all $i\leq k-1$, then $\EE\left[1+\|x_k\|^2+\|y_k\|^2\right]\leq \Gamma_3$. For the base case of $k=0$, note that
\begin{align*}
    \EE\left[1+\|x_0\|^2+\|y_0\|^2\right]&\leq 1+(4L_0^2+2)S_0+4\|\xstar\|^2+2\|\ystar\|^2\\
    &\leq \Gamma_3.
\end{align*}
Hence using the law of strong induction, $\EE\left[1+\|x_k\|^2+\|y_k\|^2\right]\leq \Gamma_3$ for all $k\geq 0$. This implies that 
$$\EE\left[\|x_k-\xstar(y_k)\|^2+\|y_k-\ystar\|^2\right]\leq \frac{D_2}{(k+K_1)^\afrak},$$
for all $k$. Now,
\begin{align*}
    \EE\left[\|x_k-\xstar\|^2\right]&\leq 2\EE\left[\|x_k-\xstar(y_k)\|^2\right]+2L_0^2\EE\left[\|y_k-\ystar\|^2\right]\\
    &\leq 2(L_0^2+1)\frac{D_2}{(k+K_1)^\afrak}\\
    &\eqqcolon \frac{D_3}{(k+K_1)^\afrak}.
\end{align*}
\subsubsection{Values of Constants in Assumption \ref{assu:stepsize-1/k^a} and Theorem \ref{thm:main-1/k^a}}
We assume $K_2\geq D_1$, where 
\begin{align*}
    D_1^{1-\afrak}&=16\cc_1(2L_0^2+1)\left(\frac{70L^2}{\lambda'\mu'^2}+\frac{4}{\lambda'}+1+L_0^2\right)+\frac{4\alpha}{\lambda'}\\ 
    &\;\;+\frac{8LL_0+64L_0^2+5+82L^2}{\lambda'\mu'}\frac{\beta}{\alpha}+\frac{\beta}{\alpha^2}+36L_0^2\beta\left(\frac{L^2}{\lambda'}+\frac{4}{\mu'}+1\right).
\end{align*}
Note that our aim here is to give just one lower bound for $K_2$ and hence we `loosely' combine multiple terms to get an expression for $D_1$. This bound can be significantly tightened by explicitly stating different terms without combining them. 

The constants $D_2$ and $D_3$ in the bound are
$$D_2=4S_0\left(K_2+\frac{4\alpha}{\lambda'}\right)+\left(\frac{280L^2}{\lambda'\mu'^2}+\frac{16}{\lambda'}+4(1+L_0^2)\right)\cc_1\Gamma_3\alpha,$$
$$D_3=2(L_0^2+1)D_2.$$
Here $\Gamma_3=2+12(4L_0^2+2)S_0+8\|\xstar\|^2+4\|\ystar\|^2$.

\section{Auxiliary Lemmas}
\begin{lemma}\label{lemma:aux}
    Suppose $\beta_k=\beta/(k+K)$. Let $\epsilon_k=\epsilon/(k+K)^{\qfrak}$, where $\qfrak\in (1,2]$. If $\pfrak\geq 0$, $\beta\geq \frac{2(\qfrak-1)}{\pfrak}$ and $\pfrak\beta_k\leq 1$, then 
    $$\sum_{i=0}^{k-1}\epsilon_i\prod_{j=i+1}^{k-1}(1-\beta_j\pfrak)\leq \frac{2}{\pfrak}\frac{\epsilon_k}{\beta_k}.$$
\end{lemma}
\begin{proof}
    Define sequence $s_{0}=0$ and $s_{k+1}=(1-\beta_k\pfrak)s_k+\epsilon_k$. Note that $s_k=\sum_{i=0}^{k-1}\epsilon_i\prod_{j=i+1}^{k-1}(1-\beta_j\pfrak)$. We will use induction to show our required result. Suppose that $s_k\leq (2/\pfrak)(\epsilon_k/\beta_k)$ holds for some $k$. Then,
    \begin{align*}
        \frac{2}{\pfrak}\frac{\epsilon_{k+1}}{\beta_{k+1}}-s_{k+1}&=\frac{2}{\pfrak}\frac{\epsilon_{k+1}}{\beta_{k+1}}-(1-\pfrak\beta_k)s_k-\epsilon_k\\
        &\geq \frac{2}{\pfrak}\frac{\epsilon_{k+1}}{\beta_{k+1}}-(1-\pfrak\beta_k)\frac{2}{\pfrak}\frac{\epsilon_k}{\beta_k}-\epsilon_k\\
        &=\frac{2}{\pfrak}\left(\frac{\epsilon_{k+1}}{\beta_{k+1}}-\frac{\epsilon_k}{\beta_k}\right)+\epsilon_k.
    \end{align*}
    Here the inequality follows from our assumption that the required inequality holds at time $k$.
    Now,
    $$\left(\frac{\epsilon_{k+1}}{\beta_{k+1}}-\frac{\epsilon_k}{\beta_k}\right)=\frac{\epsilon}{\beta}\left(\frac{1}{(k+K+1)^{\qfrak-1}}-\frac{1}{(k+K)^{\qfrak-1}}\right).$$
    For $\qfrak-1\in(0,1]$,
    \begin{align*}
        \frac{1}{(k+K+1)^{\qfrak-1}}-\frac{1}{(k+K)^{\qfrak-1}}&=\frac{1}{(k+K)^{\qfrak-1}}\left(\left[\left(1+\frac{1}{k+K}\right)^{k+K}\right]^{-\frac{\qfrak-1}{k+K}}-1\right)\\
        &\geq \frac{1}{(k+K)^{\qfrak-1}}\left(e^{-\frac{\qfrak-1}{k+K}}-1\right)\\
        &\geq -\frac{1}{(k+K)^{\qfrak-1}}\frac{\qfrak-1}{k+K}=-\frac{\qfrak-1}{\epsilon}\epsilon_k.
    \end{align*}
    Here the first inequality follows from the inequality $(1+1/x)^x\leq e$ and $e^x\geq 1+x$ for all $x$. This implies 
    \begin{align*}
        \frac{2}{\pfrak}\frac{\epsilon_{k+1}}{\beta_{k+1}}-s_{k+1}&\geq -\frac{2}{\pfrak}\frac{\epsilon}{\beta}\frac{\qfrak-1}{\epsilon}\epsilon_k+\epsilon_k\\
        &=\epsilon_k\left(1-\frac{2(\qfrak-1)}{\pfrak\beta}\right).
    \end{align*}
    Since we have the assumption that $\beta\geq \frac{2(\qfrak-1)}{\pfrak}$, $s_{k+1}\leq \frac{2}{\pfrak}\frac{\epsilon_{k+1}}{\beta_{k+1}}$. This completes the proof by induction.
\end{proof}


\end{document}